\documentclass[11pt,a4paper]{article}

\usepackage{fullpage}		 
\usepackage[utf8]{inputenc} 
\usepackage[T1]{fontenc}    
\usepackage{lmodern}			 
\usepackage{hyperref}       
\usepackage{url}            
\usepackage{booktabs}       
\usepackage{amsthm}			 
\usepackage{amsmath}			 
\usepackage{amssymb}			 
\usepackage{amsfonts}       
\usepackage{mathtools}      
\usepackage{nicefrac}       
\usepackage{microtype}      
\usepackage{subfiles}       
\usepackage{bbold}          
\usepackage{xcolor}         
\usepackage{bm}             
\usepackage{enumitem}		 
\usepackage[round]{natbib}  
\usepackage[ruled,vlined]{algorithm2e} 
\usepackage{algorithmic}

\usepackage{xpatch}
\makeatletter
\xpatchcmd{\@thm}{\thm@headpunct{.}}{\thm@headpunct{}}{}{}
\makeatother

\makeatletter
\def\th@plain{%
  \thm@notefont{}
  \itshape 
}
\def\th@definition{%
  \thm@notefont{}
  \normalfont 
}
\makeatother

\newtheorem{theorem}{Theorem}

\newtheorem{lemma}[theorem]{Lemma}

\newtheorem{remark}{Remark}

\definecolor{myred}{HTML}{880000}
\definecolor{mygreen}{HTML}{008800}
\definecolor{myblue}{HTML}{000088}
\definecolor{linkblue}{HTML}{0000BB}

\hypersetup{
  colorlinks,
  linkcolor={myred},
  citecolor={myblue},
  urlcolor={linkblue}
}

\newcommand{\R}{\mathbb{R}}
\newcommand{\E}{\mathbb{E}}
\newcommand{\x}{\mathbf{x}}
\newcommand{\A}{\mathbf{A}}
\newcommand{\Eins}{\mathrm{\textbf{1}}}
\newcommand{\gauss}{\mathcal{N}}
\renewcommand{\O}{\mathcal{O}}
\renewcommand{\P}{\mathbb{P}}

\title{Hadamard Wirtinger Flow for Sparse Phase Retrieval}
\author{
  Fan Wu,
  Patrick Rebeschini \\
	Department of Statistics, University of Oxford
}

\begin{document}

\maketitle

\begin{abstract}%
	We consider the problem of reconstructing an $n$-dimensional $k$-sparse signal from a set of noiseless magnitude-only measurements. Formulating the problem as an unregularized empirical risk minimization task, we study the sample complexity performance of gradient descent with Hadamard parametrization, which we call Hadamard Wirtinger flow (HWF). Provided knowledge of the signal sparsity $k$, we prove that a single step of HWF is able to recover the support from $k(x^*_{max})^{-2}$ (modulo logarithmic term) samples, where $x^*_{max}$ is the largest component of the signal in magnitude. This support recovery procedure can be used to initialize existing reconstruction methods and yields algorithms with total runtime proportional to the cost of reading the data and improved sample complexity, which is linear in $k$ when the signal contains at least one large component. We numerically investigate the performance of HWF at convergence and show that, while not requiring any explicit form of regularization nor knowledge of $k$, HWF adapts to the signal sparsity and reconstructs sparse signals with fewer measurements than existing gradient based methods. 
\end{abstract}

\section{Introduction}
\label{section:introduction}
Phase retrieval, the problem of reconstructing a signal from the (squared) magnitude of its Fourier (or any linear) transform, arises in many fields of science and engineering. Such a task is naturally involved in applications such as crystallography \citep{M90} and diffraction imaging \citep{BDPFSSV07}, where optical sensors are able to measure intensities, but not phases of  light waves. Due to the loss of phase information, the one-dimensional Fourier phase retrieval problem is ill-posed in general. Common approaches to overcome this ill-posedness include using prior information such as non-negativity, sparsity and the signal's magnitude \citep{F82, JEH16}, or introducing redundancy into the measurements by oversampling random Gaussian measurements or coded diffraction patterns \citep{CLS15, CC15}.

In many applications, the underlying signal is naturally sparse \citep{JEH16}. A wide range of algorithms has been devised for phase retrieval with a sparse signal, including alternating minimization (SparseAltMinPhase) \citep{NJS15}, non-convex optimization based approaches such as thresholded Wirtinger flow (TWF) \citep{CLM16}, sparse truncated amplitude flow (SPARTA) \citep{WZGAC18}, compressive reweighted amplitude flow (CRAF) \citep{ZWGC18} and sparse Wirtinger flow (SWF) \citep{YWW19}, and convex relaxation based methods such as compressive phase retrieval via lifting (CPRL) \citep{OYDS12} and SparsePhaseMax \citep{HV16}. Other approaches to sparse phase retrieval include the greedy algorithm GESPAR \citep{SBE14}, a generalized approximate message passing algorithm (PR-GAMP) \citep{SR15} and majorization minimization algorithms \citep{QP17}. 

A limitation of these algorithms is their sample complexity: the best known theoretical results require $\O(k^2\log n)$ Gaussian measurements to guarantee successful reconstruction of a generic $k$-sparse signal $\mathbf{x}^*\in\R^n$. On the other hand, it has been shown that reconstruction is possible from $\O(k\log n)$ phaseless measurements \citep{EM14}; however, there is no known algorithm which provably achieves this in polynomial time. In fact, $\O(k^2 \log n)$ quadratic measurements are necessary for a certain class of convex relaxations, on which algorithms such as CPRL are based \citep{LV13}. Other existing algorithms such as SPARTA and SparseAltMinPhase require a sample complexity of $\O(k^2\log n)$ for the initial estimation of the support of the signal. With the knowledge of the support, these (and plenty other) algorithms require only $\O(k\log n)$ samples for the subsequent reconstruction of the signal. Hence, we identify the support recovery step as the bottleneck of the sample complexity of non-convex optimization based approaches to the sparse phase retrieval problem.

Additional structural assumptions have been considered to improve the sample complexity. It has been shown that a $k$-sparse signal $\mathbf{x}^*$ can be reconstructed from $\O(k\log n)$ measurements if one is allowed to freely design the measurement vectors \citep{JOH13}, or if the signal $\mathbf{x}^*$ is assumed to be block-sparse and the number of blocks containing non-zero entries is $\O(1)$ \citep{JH17, ZWGC18}. However, exact knowledge of the additional structure as well as an algorithm designed to take advantage of it is necessary in both cases. Linear sample complexity has also been achieved assuming that the signal coefficients decay with power-law \citep{JH19}.

Another downside of the above algorithms is the fact that sparsity is enforced or promoted \textit{explicitly}. For instance, CPRL and SparsePhaseMax augment the objective function with an $\ell_1$ penalty term, which is known to promote sparsity. SWF and SPARTA include a thresholding step in their gradient updates, which projects the iterates onto the set of $k$-sparse vectors, and SparseAltMinPhase and GESPAR directly constrain the search to a $k$-dimensional subspace of $\R^n$, which needs to be carefully chosen and updated. In the case of CPRL and SparsePhaseMax, additional regularization parameters have to be tuned, while the thresholding step of SWF and SPARTA requires knowledge of the sparsity $k$.

\subsection{Our Contributions}
In this work, we analyze gradient descent with Hadamard parametrization applied to the unregularized empirical risk for the problem of (noiseless) sparse phase retrieval and propose methods for support recovery and parameter estimation. The main contributions of this paper are stated below.

First, we propose a two-stage procedure for sparse phase retrieval, which we call \textit{Hadamard Wirtinger flow} (HWF) (following the terminology used for Wirtinger flow \citep{CLS15}, which considers gradient descent applied to the unregularized empirical risk under the natural parametrization to solve phase retrieval without the assumption on sparsity).
In stage one, we estimate a \textit{single} coordinate on the support to construct a simple initial estimate, without using a sophisticated initialization scheme typically required such as the spectral initialization used in WF and SWF or the orthogonality-promoting initialization used in SPARTA.
For stage two, we consider the Hadamard parametrization, which has previously been applied to problems on sparse recovery \citep{H17, VKR19, ZYH19} and matrix factorization \citep{GWBNS17, LMZ18, ACHL19}, and apply gradient descent to the unregularized empirical risk under this parametrization.

Second, we prove that our proposed algorithm can be used to recover the support $\mathcal{S} = \{i:x^*_i\neq 0\}$ with high probability by choosing the $k$ largest components of the estimate obtained after one step of HWF, provided that $m\ge \O(\max\{k\log n,\, \log^3n\}(x^*_{max})^{-2})$ and $x^*_{min} = \Omega(1/\sqrt{k})$, where we write $x^*_{max} := \max_i |x^*_i|/\|\mathbf{x}^*\|_2$ and $x^*_{min} := \min_{i\in \mathcal{S}}|x^*_i|/\|\mathbf{x}^*\|_2$. Note that HWF does not require knowledge of the sparsity level $k$, while support recovery using one step of HWF does. With the knowledge of the support $\mathcal{S}$, plenty of algorithms provably recover the signal $\mathbf{x}^*$ under linear sample complexity $\O(k\log n)$, see e.g.\ \citep{CLS15, WZGAC18}. Thus, provided knowledge of the sparsity level $k$, one step of HWF can be used as a support recovery tool and, combined with any of the aforementioned algorithms, it results in a procedure which provably recovers $k$-sparse signals from $\O(\max\{k\log n,\, \log^3n\}(x^\star_{max})^{-2})$ phaseless measurements. 

If $x^*_{max}=\Omega(1)$, then the sample complexity of this procedure reduces to $\O(k\log n)$, provided $k\ge \O(\log^2n)$. Unlike previous results which leverage additional structural assumptions to achieve linear sample complexity, our procedure does not require knowledge of the value of $x^*_{max}$, nor does it need to be modified in any way to accommodate for this additional structure; we run the exact same algorithm regardless of the value of $x^*_{max}$.

Third, we present numerical experiments showing the low sample complexity of HWF. As a simple algorithm not requiring thresholding steps nor added regularization terms to promote sparsity, HWF is seen to \textit{adapt} to the sparsity level of the underlying signal and to reconstruct signals from a similar number of Gaussian measurements as PR-GAMP, which has been empirically shown to achieve linear sample complexity in some regimes with Gaussian signals \citep{SR15}. In particular, the numerical experiments suggest that the sample complexity required by HWF is lower than that of other gradient based methods such as SPARTA and SWF. Further, if the signal satisfies $x^*_{max}=\Omega(1)$, the reconstruction performance of HWF is seen to be greatly improved, without any modifications to the algorithm being made.

\section{Sparse Phase Retrieval}
\label{section:background}
We denote vectors and matrices with boldface letters and real numbers with normal font, and, where appropriate, use uppercase letters for random and lowercase letters for deterministic quantities. For vectors $\mathbf{u},\mathbf{v} \in \R^n$ we write $\odot$ for the Hadamard product, $(\mathbf{u}\odot \mathbf{v})_i = u_iv_i$, and, for notational simplicity, $\mathbf{u}^2 = \mathbf{u}\odot \mathbf{u}$ for taking squares entry-wise. We use the common notation $[n] := \{1,...,n\}$. Since it is impossible to distinguish $\mathbf{x}^*$ from $-\mathbf{x}^*$ using magnitude-only observations, we will often write $\mathbf{x}^*$ for the solution set $\{\pm\mathbf{x}^*\}$ and consider, for any $\mathbf{x}\in\R^n$, the distance $\operatorname{dist}(\mathbf{x},\mathbf{x}^*) :=  \min\{\|\mathbf{x}-\mathbf{x}^*\|_2, \|\mathbf{x}+\mathbf{x}^*\|_2\}$. Further, we assume $\|\mathbf{x}^*\|_2=1$ for notational simplicity; note that this assumption is not needed for our results.

The goal in phase retrieval is to reconstruct an unknown signal vector $\mathbf{x}^*\in \R^n$ from a set of quadratic measurements $Y_j = (\mathbf{A}_j^T\mathbf{x}^*)^2$, $j = 1,...,m,$ where we observe $\mathbf{A}_j \sim \mathcal{N}(0,\mathbf{I}_n)$ i.i.d.. For the sake of clarity, we focus on the real-valued model. Our proposed algorithm also works in the complex-valued Gaussian model, where $\mathbf{x}^*\in\mathbb{C}^n$ and $\mathbf{A}_j\sim \mathcal{N}(0,\frac{1}{2}\mathbf{I}_n) + i\mathcal{N}(0,\frac{1}{2}\mathbf{I}_n)$.

Many methods have been devised to solve this problem. A popular class of algorithms performs alternating projections onto different constraint sets; these include the seminal error reduction algorithm proposed by \cite{GS72} and alternating minimization (AltMinPhase) \citep{NJS15}. Another more recent approach is based on non-convex optimization: Wirtinger flow (WF) \citep{CLS15} and its variants \citep{CC15, ZZLC17}, truncated amplitude flow (TAF) \citep{WGE17} and the trust region method of \citep{SQW18} all minimize the empirical risk (which is non-convex due to the missing phase) based on different loss functions. The convex alternatives typically use matrix-lifting as in PhaseLift \citep{CL12, CSV13} and PhaseCut \citep{WAM15}, which allows phase retrieval to be formulated as a semidefinite programming problem, or consider a non-lifting convex relaxation and solve the dual problem as in PhaseMax \citep{GS18}. 

Our approach for estimating the signal $\mathbf{x}^*$ follows the established approach of empirical risk minimization. Writing $\mathbf{z} = (y,\mathbf{a})\in\R\times\R^n$ for an observation, we consider the loss function $\ell(\mathbf{x},\mathbf{z}) =\frac{1}{4}((\mathbf{a}^T\mathbf{x})^2-y)^2$ and, given samples $\mathbf{Z}_1,...,\mathbf{Z}_m$, the empirical risk
\begin{equation}\label{eq:risk}
F(\mathbf{x}) = \frac{1}{4m}\sum_{j=1}^m \left((\mathbf{A}_j^T\mathbf{x})^2 - Y_j\right)^2.
\end{equation}
It is worth mentioning that in previous applications the amplitude-based loss function $\ell(\mathbf{x},\mathbf{z}) =\frac{1}{2}(|\mathbf{a}^T\mathbf{x}|-\sqrt{y})^2$ has been numerically shown to be more effective in terms of sample complexity than the loss function based on squared magnitudes \citep{WGE17, ZZLC17, WZGAC18}. However, with our parametrization, we found the squared magnitude-based loss function to be more effective. 

Without any restrictions on the signal $\mathbf{x}^*\in\R^n$, $2n-1$ Gaussian measurements suffice for $\mathbf{x}^*$ to be the unique (up to global sign) minimizer of $F(\mathbf{x})$ with high probability \citep{BCE06}. If $\mathbf{x}^*$ is $k$-sparse, then it has been shown in \citep{LV13} that
\begin{equation}\label{eq:problem1}
\{\pm\mathbf{x}^*\} = \underset{\mathbf{x}:\|\mathbf{x}\|_0\le k}{\operatorname{argmin}} F(\mathbf{x})
\end{equation}
holds with high probability if we have $m \ge 4k-1$ Gaussian measurements.

Solving (\ref{eq:problem1}) involves two main difficulties: $(i)$ the objective function $F$ is non-convex with potentially many local minima and saddle points, and $(ii)$ due to the constraint $\|\mathbf{x}\|_0\le k$ the problem is of combinatorial nature and NP-hard in general. 

Regarding the first difficulty, the non-convexity is typically addressed by using a spectral or orthogonality-promoting initialization, which produces an initial estimate close to $\mathbf{x}^*$ and in a region where the the objective function is locally strongly convex, leading to linear convergence towards $\mathbf{x}^*$ \citep{CLS15, CC15, WGE17, ZZLC17}. Recently, it has been shown that such an initialization is not always necessary in the phase retrieval problem and that a random initialization can be used instead \citep{SQW18, CCFM19}. 

Addressing the second difficulty, an approach replacing the constraint $\|\mathbf{x}\|_0\le k$ in (\ref{eq:problem1}) by adding a penalty term $\lambda \|\mathbf{x}\|_1$ was proposed in \citep{YZX13}. However, this procedure requires tuning of the regularization parameter $\lambda$ to reach a desired sparsity level, and it is tailored for Fourier measurements only (in particular, it uses the fact that the DFT matrix is unitary). Recently, methods enforcing the constraint $\|\mathbf{x}\|_0\le k$ via a hard-thresholding step have received a lot of attention. These include SPARTA \citep{WZGAC18}, CRAF \citep{ZWGC18} and SWF \citep{YWW19}. However, the implementation of such a thresholding step requires knowledge of $k$ (or a suitable upper bound).

Our proposed method does not have to deal with these difficulties. We also approach the problem by minimizing the objective $F(\mathbf{x})$. However, unlike existing algorithms, we do not need to add any penalty term to the objective or to introduce a thresholding step to enforce the constraint $\|\mathbf{x}\|_0\le k$. Our simulations show that the iterates of gradient descent with Hadamard parametrization remain (approximately) in the low-dimensional space of sparse vectors. Hence, we neither need to tune any regularization parameters, nor do we need knowledge of the underlying signal sparsity. Further, HWF does not need the sophisticated initialization scheme commonly used in non-convex optimization based approaches to (sparse) phase retrieval.
\section{Hadamard Wirtinger Flow}
\label{section:hwf}
Consider the parametrization $\mathbf{x} = \mathbf{u}\odot \mathbf{u} - \mathbf{v}\odot \mathbf{v}$. Such a parametrization has previously been applied to problems such as sparse recovery \citep{H17, VKR19, ZYH19} and matrix factorization \citep{GWBNS17, LMZ18, ACHL19}. Exploiting the restricted isometry property (RIP) assumed for these problems, the Hadamard parametrization has been shown to confine the gradient iterates to the low-dimensional spaces of sparse vectors and low-rank matrices, respectively.

Overloading the notation, we write 
\begin{equation*}
F(\mathbf{u},\mathbf{v}) = \frac{1}{4m}\sum_{j=1}^m \big((\mathbf{A}_j^T(\mathbf{u}^2-\mathbf{v}^2))^2 - Y_j\big)^2,
\end{equation*}
for the empirical risk, with gradients 
\begin{align*}
\nabla_{\mathbf{u}} F(\mathbf{u},\mathbf{v})  &= \frac{2}{m}\sum_{j=1}^m\big((\mathbf{A}_j^T(\mathbf{u}^2-\mathbf{v}^2))^2 - (\mathbf{A}_j^T\mathbf{x}^*)^2\big) (\mathbf{A}_j^T(\mathbf{u}^2-\mathbf{v}^2))\mathbf{A}_j \odot \mathbf{u} \\
 &= 2\nabla F(\mathbf{x}) \odot \mathbf{u},
\end{align*}
and, similarly, $\nabla_{\mathbf{v}} F(\mathbf{u},\mathbf{v}) = -2\nabla F(\mathbf{x})\odot \mathbf{v}$. We consider gradient descent in this parametrization,
\begin{equation}
\label{eq:hwf}
\begin{gathered}
\mathbf{X}^t = \mathbf{U}^t\odot \mathbf{U}^t - \mathbf{V}^t\odot \mathbf{V}^t, \\
\mathbf{U}^{t+1} = \mathbf{U}^t\odot \big(\Eins_n - 2\eta\nabla F(\mathbf{X}^t)\big), \\
 \mathbf{V}^{t+1} = \mathbf{V}^t\odot \big(\Eins_n + 2\eta\nabla F(\mathbf{X}^t)\big),
\end{gathered}
\end{equation}
where we denote by $\Eins_n\in \R^n$ the vector of all ones.
 
The reason why the Hadamard parametrization promotes sparsity is that this parametrization turns the additive updates of gradient descent into multiplicative updates. If we choose a small initialization, then, in the aforementioned problems with RIP assumptions, the multiplicative updates have been shown to lead to off-support variables staying negligibly small while support variables are being fitted. The variables grow exponentially, but at different (time-varying) rates, with off-support variables growing at a smaller rate than support variables. With additive updates, off-support variables would not stay sufficiently small and the algorithm would typically converge towards non-sparse local minima.
We provide a more detailed discussion on the role of the Hadamard parametrization in the appendix, which suggests the following initialization:
\begin{align}
\label{eq:initialization}
\mathbf{V}^0 = \alpha \Eins_n, \qquad U^0_i = \begin{cases} \Bigl(\frac{\hat{\theta}}{\sqrt{3}} + \alpha^2\Bigr)^{\frac{1}{2}} &i=I_{max}\\
\alpha & i\neq I_{max}
\end{cases}
\end{align}
where we write $\hat{\theta} = (\frac{1}{m}\sum_{j=1}^mY_j)^{1/2}$ for the estimate of the signal size $\|\x^*\|_2$ (see e.g.\ \cite{CLS15, WGE17}), $I_{max} = \operatorname{argmax}_{i} \sum_{j=1}^mY_j A_{ji}^2$ and $\alpha>0$ is the initialization size. 

We show in the next section that if $|x^*_{I_{max}}|\ge \frac{1}{2}x^*_{max}$ and $m$ is sufficiently large, then, with high probability, we can recover the support by running one step of (\ref{eq:hwf}).
The probability of finding a large coordinate $I_{max}$ can be increased by allowing multiple restarts and considering not only the largest, but also a few more instances in $\{R_i\}_{i=1}^n := \{\frac{1}{m}\sum_{j=1}^mY_jA_{ji}^2\}_{i=1}^n$. Specifically, if we allow $\bar{b}$ restarts, we consider different initialization as in (\ref{eq:initialization}) using each of the $\bar{b}$ largest instances in $\{R_i\}_{i=1}^n$.
As pointed out in Section \ref{section:background}, $\mathbf{x}^*$ is the sparsest minimizer of the objective $F$. Given the results from multiple runs, we can therefore choose the (approximately) sparsest solution, by which we mean the solution where the fewest coordinates make up most (e.g.\ $95\%$) of the norm $\|\mathbf{X}^{\bar{t}}\|_2$. This is summarized in Algorithm \ref{alg:wfrmr}.

\begin{algorithm}[h]
   \caption{Hadamard Wirtinger flow, $\bar{b}$ restarts}
   \label{alg:wfrmr}
\begin{algorithmic}
   \STATE {\bfseries Input:} observations $\{Y_j\}_{j=1}^m$, measurement vectors $\{\mathbf{A}_j\}_{j=1}^m$, step size $\eta$, iterations $\bar{t}$, initialization size $\alpha$, number of restarts $\bar{b}$, sparsity tolerance $\kappa$
	 \STATE \vspace{-3mm}
   \FOR{$b=1$ {\bfseries to} $\bar{b}$}
	 \STATE Set $I_b$ to the $b^{th}$ largest instance in $\{R_i\}_{i=1}^n$
	 \STATE \vspace{-0mm} Set \hspace{1.5mm} $\mathbf{U}^{0,b} = \mathbf{V}^{0,b} = \alpha \Eins_n$,\hspace{3mm} $U^{0,b}_{I_b} = \Big(\frac{\hat{\theta}}{\sqrt{3}} + \alpha^2\Big)^{\frac{1}{2}}$ \vspace{-0.7mm}
	 \FOR{$t=0$ {\bfseries to} $\bar{t}$}
		\STATE \vspace{-6.5mm}
		\begin{align*}
		&\mathbf{X}^{t,b} = \mathbf{U}^{t,b} \odot \mathbf{U}^{t,b} - \mathbf{V}^{t,b} \odot \mathbf{V}^{t,b} \\
		&\mathbf{U}^{t+1,b} = \mathbf{U}^{t,b} \odot \big(\Eins_n - 2\eta\nabla F(\mathbf{X}^{t,b})\big) \\
				&\mathbf{V}^{t+1,b} = \mathbf{V}^{t,b} \odot \big(\Eins_n + 2\eta\nabla F(\mathbf{X}^{t,b})\big)
		\end{align*}
		\vspace{-6.5mm}
	 \ENDFOR
   \ENDFOR
	 \STATE Set $B_{min}$ to be the index minimizing
	\begin{equation*}
	\underset{\mathcal{C}\subset [n]}{\min}\Biggl\{|\mathcal{C}|:\sum_{i\in C} (\mathbf{X}^{\bar{t},b}_i)^2 \ge (1-\kappa)\sum_{i=1}^n (\mathbf{X}^{\bar{t},b}_i)^2\Biggr\}
	\end{equation*}
	 \STATE {\bfseries Return:} $\mathbf{X}^{\bar{t},B_{min}}$
\end{algorithmic}
\end{algorithm}

\section{Support recovery} 
\label{section:supportrecovery}
In this section we show, assuming $x^*_{min}=\Omega(1/\sqrt{k})$, that one step of Algorithm \ref{alg:wfrmr} can be used to recover the support $\mathcal{S}$ from $\O(\max\{k\log n, \, \log^3n\}(x^*_{max})^{-2})$ Gaussian measurements. As pointed out in Section \ref{section:background}, the sample complexity bottleneck of reconstruction algorithms such as SPARTA and SparseAltMinPhase lies in the initial support recovery. In particular, both algorithms require $\O(k^2\log n)$ measurements to guarantee successful support recovery; with the knowledge of the support $\mathcal{S}$, these and plenty other algorithms such as WF, TAF and PhaseLift only require a sample complexity of $\O(k\log n)$ to guaruantee successful reconstruction of a $k$-sparse signal $\mathbf{x}^*$. 
  
The main difference from previous work is that we only need a single coordinate $i$ with $|x^*_i|\ge \frac{1}{2}x^*_{max}$ rather than the full support for our initialization. Therefore, our sample complexity depends on $x^*_{max}$, which is at least $1/\sqrt{k}$, rather than $x^*_{min}$, which can be at most $1/\sqrt{k}$. This is made precise in the following Lemma.
\begin{lemma}(Support recovery)\label{lemma:suprec} Let $\mathbf{x}^*\in \R^n$ be any $k$-sparse vector with $x^*_{min} = \Omega(1/\sqrt{k})$, and assume that we are given measurements $\{Y_j = (\mathbf{A}_j^T\mathbf{x}^*)^2\}_{j=1}^m$, where $\mathbf{A}_j\sim\mathcal{N}(0,\mathbf{I}_n)$, $j = 1,\dots, m$, are i.i.d.\ Gaussian vectors. If $m\ge \O(\max\{k\log n,\, \log^3n\}(x^*_{max})^{-2})$, then, with probability at least $1-\O(n^{-10})$, choosing the largest instance in $\{\frac{1}{m}\sum_{j=1}^m Y_jA_{ji}^2\}_{i=1}^n$ returns an index $i$ with $|x^*_i| \ge \frac{1}{2}x^*_{max}$.\\
Further, let $\mathbf{X}^1$ be the estimate obtained from running one step of Algorithm \ref{alg:wfrmr} with any $\eta,\alpha\in (0,\frac{1}{10})$ (and $\bar{b}=1$, i.e. no multiple restarts). Then, with the same probability, we can recover the  support $\mathcal{S}=\{i:x^*_i\neq 0\}$ by choosing the $k$ largest coordinates of $|\mathbf{X}^1|$ (where $|\cdot|$ denotes taking absolute values coordinate-wise).
\end{lemma}
The proof of Lemma \ref{lemma:suprec} relies on standard concentration results and is deferred to the appendix.

Lemma \ref{lemma:suprec} shows that, provided knowledge of the sparsity level $k$, Algorithm \ref{alg:wfrmr} can be used to recover the support of a $k$-sparse signal $\mathbf{x}^*$ from $\O(\max\{k\log n,\, \log^3n\}(x^*_{max})^{-2})$ Gaussian measurements, which, provided $k\ge \O(\log^2n)$, matches the best known bounds $\O(k^2\log n)$ \citep{NJS15, WZGAC18} in the worst case, while it is an improvement if $\mathbf{x}^*$ contains 
(at least) one large coordinate. For instance, if $x^*_{max}=\Omega(1)$ and $k\ge \O(\log^2n)$, only $\O(k\log n)$ samples are required for support recovery.

We validate this theoretical result in the following experiment. Let $\mathbf{x}^*\in \R^{10000}$ be a $k$-sparse signal with randomly sampled support $\mathcal{S}=\{i_1,...,i_k\}$ and normalized to $\|\mathbf{x}^*\|_2=1$. We consider maximum signal values $(i)$ $x^*_{max} = \frac{1}{\sqrt{k}}$, $(ii)$ $x^*_{max} = k^{-0.25}$ and $(iii)$ $x^*_{max} = 0.7$. In case $(i)$, we set $x^*_{i_j} = \pm \frac{1}{\sqrt{k}}$ at random for all $j=1,...,k$. For the cases $(ii)$ and $(iii)$, we fix $x^*_{i_1} = x^*_{max}$, sample the other components from $x^*_{i_j}\sim\mathcal{N}(0,1)$ i.i.d., and then normalize them to satisfy $\|\mathbf{x}^*\|_2=1$. We also consider a signal with $x^*_{i_j}\sim \mathcal{N}(0, 1)$ i.i.d.\ normalized to $\|\mathbf{x}^*\|_2=1$, without any restrictions on $x^*_{max}$. We generate $m=5000$ measurements $Y_j = (\mathbf{A}_j^T\mathbf{x}^*)^2$ with $\mathbf{A}_j\sim \mathcal{N}(0, \mathbf{I}_n)$ i.i.d.. 

For our method (HWF) we run one step of Algorithm \ref{alg:wfrmr} and pick the $k$ largest components of $|\mathbf{X}^1|$. We compare it with the support recovery methods used in SPARTA and SparseAltMinPhase (other algorithms like SWF and CRAF use the same support recovery method as SPARTA). Note that although correct identification of the full support is required for the theoretical guarantees of algorithms like SPARTA, it is not necessary in practice: it has been noted in \citep{WZGAC18} that, since the estimated support $\hat{\mathcal{S}}$ is only used for the orthogonality-promoting initialization, SPARTA can be successful as long as the initial estimate is sufficiently close to the underlying signal (more precisely, $\operatorname{dist}(\mathbf{X}^0, \mathbf{x}^*)\le \frac{1}{10}\|\mathbf{x}^*\|_2$), regardless of whether or not the full support has been correctly identified. Intuitively, the initialization produces an estimate sufficiently close to $\mathbf{x}^*$ as long as the majority of the support is recovered. This intuition has been made rigorous for an alternative spectral initialization in \citep{JH17}.

We evaluate the proportion of correctly recovered support variables
\begin{equation*}
\frac{|\hat{\mathcal{S}}\cap \mathcal{S}|}{|\mathcal{S}|}
\end{equation*}
obtained from 100 independent Monte Carlo trials, where $\hat{\mathcal{S}}\subset [n]$ denotes the estimated support.

\begin{figure}[t]
\begin{center}
\centerline{\includegraphics[width=\columnwidth]{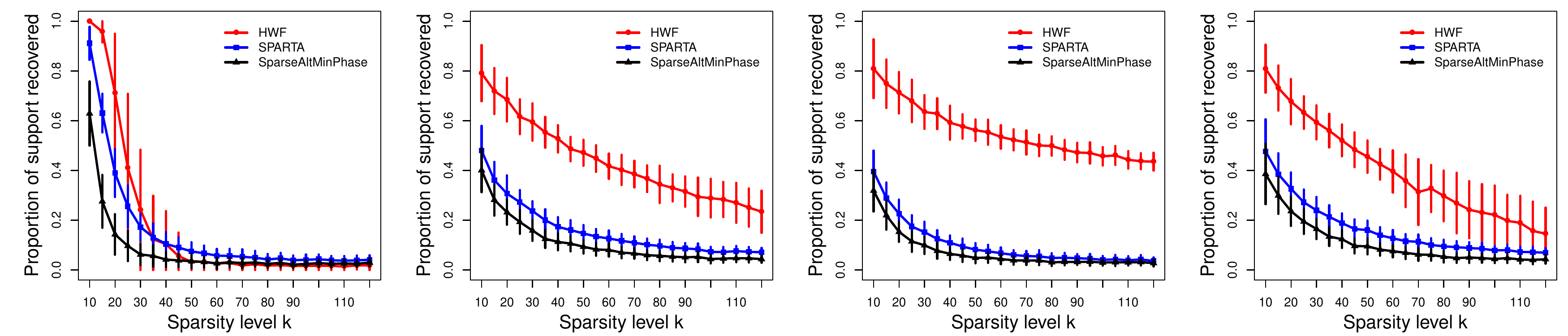}}
\caption{Proportion of support variables $|\hat{\mathcal{S}}\cap \mathcal{S}|/|\mathcal{S}|$ correctly recovered plus/minus one standard deviation (vertical lines) by our method (red curve), SPARTA (blue curve) and SparseAltMinPhase (black curve), for different levels of $x^*_{max}$. From left to right: $(i)$ $x^*_{max} = \frac{1}{\sqrt{k}}$, $(ii)$ $x^*_{max} = k^{-0.25}$, $(iii)$ $x^*_{max} = 0.7$, $(iv)$ Gaussian signal $\mathbf{x}^*$ (without restrictions on $x^*_{max}$).} 
\label{fig:support}
\end{center}
\vskip -0.4in
\end{figure}
Figure \ref{fig:support} confirms the predictions of Lemma \ref{lemma:suprec}. In the first case, where the signal only takes the values $x^*_i\in \{-\frac{1}{\sqrt{k}}, 0, \frac{1}{\sqrt{k}}\}$, our theoretical bound reads $\O(k^2\log n)$ (as we are considering a regime where $k\ge \O(\log^2n)$), and we expect the support recovery performance of our method to be comparable to the other methods. For sparsity levels $k\ge 45$, our method is slightly worse than the other two, because it sometimes fails to identify a coordinate $i$ with $x^*_i\neq 0$ in the first step. This case is of little practical relevance, since if only such a small portion of the support is recovered, neither of the three algorithms is able to reconstruct $\mathbf{x}^*$. As we increase $x^*_{max}$, the performance of our method improves substantially, while the support recovery methods used in SPARTA and SparseAltMinPhase do not show any improvement (in fact, they get slightly worse, which can be attributed to the fact that, as we increase $x^*_{max}$, the other coordinates become smaller since we keep $\|\mathbf{x}^*\|_2=1$ fixed). Our method also shows better support recovery performance for Gaussian signals $\mathbf{x}^*$, where we do not fix $x^*_{max}$ (bottom-right figure). 
\section{Parameter estimation}
\label{section:paramterestimation}
In this section, we first demonstrate that our support recovery method, one step of HWF, can be combined with existing algorithms (such as SPARTA), which leads to a final procedure which provably recovers a $k$-sparse signal from $\O(\max\{k\log n,\, \log^3n\}(x^*_{max})^{-2})$ measurements; this is summarized in Algorithm \ref{alg:spartasup}.
\vskip -0.05in
\begin{algorithm}[h]
   \caption{SPARTA-support, multiple restarts}
   \label{alg:spartasup}
\begin{algorithmic}
   \STATE {\bfseries Input:} observations $\{Y_j\}_{j=1}^m$, measurement vectors $\{\mathbf{A}_j\}_{j=1}^m$, sparsity level $k$, step size $\eta$, iterations $\bar{t}$, initialization size $\alpha$, number of restarts $\bar{b}$, parameters for SPARTA specified \\ in \citep{WZGAC18}
	 \STATE \vspace{-3mm}
   \FOR{$b=1$ {\bfseries to} $\bar{b}$}
	 \STATE Set $I_b$ to the $b^{th}$ largest instance in $\{R_i\}_{i=1}^n$
	 \STATE \vspace{-0mm} Set $\mathbf{U}^0 = \mathbf{V}^0 = \alpha \Eins_n$,\hspace{1mm} $U^0_{I_b} = \Big(\frac{\hat{\theta}}{\sqrt{3}} + \alpha^2\Big)^{1/2}$
	 \STATE Run one step of HWF (\ref{eq:hwf}) for $\mathbf{X}^{1, b}$
	\STATE Set $\hat{\mathcal{S}}_b$ to the $k$ largest coordinates of $|\mathbf{X}^{1,b}|$
	\STATE Run $\bar{t}$ iterations of SPARTA using $\hat{\mathcal{S}}_b$ for $\mathbf{X}^{\bar{t},b}$
	
   \ENDFOR
	 \STATE Set $B_{min}$ to be the index minimizing
	$\|\nabla F(\mathbf{X}^{\bar{t},b})\|_2$
	 \STATE {\bfseries Return:} $\mathbf{X}^{\bar{t}, B_{min}}$
\end{algorithmic}
\end{algorithm}

Note that we can allow multiple restarts in this case as well ($\bar{b}>1$) in order to further improve the probability of obtaining a good initialization. Since SPARTA only produces $k$-sparse solutions due to the thresholding step, we choose the final solution by selecting the one which produces the smallest gradient $\|\nabla F(\mathbf{X}^{\bar{t},b})\|_2$.

Our analysis from the previous section immediately leads to the following result.
\begin{theorem}\label{thm} Let $\mathbf{x}^*\in \R^n$ be any $k$-sparse vector with $x^*_{min} = \Omega(1/\sqrt{k})$, and assume that we are given measurements $\{Y_j = (\mathbf{A}_j^T\mathbf{x}^*)^2\}_{j=1}^m$, where \mbox{$\mathbf{A}_j\sim\mathcal{N}(0,\mathbf{I}_n)$}, $j= 1,\dots, m$, are i.i.d.\ Gaussian vectors. If \mbox{$m\ge \O(\max\{k\log n,\log^3n\}(x^*_{max})^{-2})$}, then, with the parameters specified in \citep{WZGAC18}, successive estimates of SPARTA-support satisfy, with probability at least $1 - \O(m^{-1} + n^{-10})$ and for a universal constant $0<v<1$,
\begin{equation*}
\operatorname{dist}(\mathbf{X}^t, \mathbf{x}^*) \le \frac{1}{10}(1-\nu)^t \|\mathbf{x}^*\|_2, \qquad t\ge 0.
\end{equation*}
\end{theorem}
\begin{proof}
By Lemma \ref{lemma:suprec}, one step of HWF recovers the true support with probability $1-\O(n^{-10})$. The result then follows from Lemma 2 and 3 of \citep{WZGAC18}.
\end{proof}
Compared to Theorem 1 of \citep{WZGAC18}, this result reduces the sample complexity from $\O(k^2\log n)$ to $\O(k(x^*_{max})^{-2}\log n)$, provided $k\ge \O(\log^2n)$. The assumption $x^*_{min}=\Omega(1/\sqrt{k})$ is likely an artifact of the proof method of \citep{WZGAC18} and not necessary. Intuitively, identifying the full support is not necessary, as a good initialization can also be obtained if only small coordinates with $x^*_i\le \O(1/\sqrt{k})$ are missed. This intuition has been made rigorous for an alternative spectral initialization \citep{JH17}. However, the orthogonality-promoting initialization used in SPARTA has been experimentally found to produce an initial estimate closer to the signal $\x^*$ than the spectral initialization \citep{WGE17, ZWGC18}.

One step of Algorithm \ref{alg:wfrmr} requires $\O(nm)$ operations, and $\bar{t}=\O(\log(1/\epsilon))$ SPARTA iterations are sufficient to find an $\epsilon$-accurate solution, so SPARTA-support incurs a total computational cost of $\O(nm\log(1/\epsilon))$. This is proportional to the cost of reading the data modulo logarithmic terms. 

However, SPARTA-support enforces sparsity of the estimates $\mathbf{X}^t$ explicitly via a hard-thresholding step, which requires knowledge of $k$ (or an upper bound). Our simulations show that HWF \textit{adapts} to the signal sparsity $k$: we neither need knowledge of $k$ for thresholding steps, nor do we need to add a penalty term to the objective and tune regularization parameters to promote sparsity. Given enough samples, our algorithm automatically converges to the $k$-sparse signal $\mathbf{x}^*$.

In the following, we present simulations evaluating the reconstruction performance of HWF and SPARTA-support relative to state-of-the-art methods for sparse phase retrieval. In particular, we will consider SPARTA, SWF and PR-GAMP.
\begin{remark}[Comparison with PR-GAMP]
Our numerical experiments show comparable sample complexities for PR-GAMP and HWF, with both being lower than the sample requirement of other gradient-based methods. PR-GAMP has been empirically shown to achieve linear sample complexity in some regimes with Gaussian signals \citep{SR15}. However, PR-GAMP relies on the implementation and tuning of several algorithmic principles, such as damping, normalization, and expectation-maximization (EM) steps. On the one hand, the application of these algorithmic principles makes PR-GAMP difficult to analyze, as rigorous theoretical investigations are known to be challenging even for much simpler AMP-based algorithms \citep{BM11}. On the other hand, running PR-GAMP requires tuning of several parameters, including the sparsity rate $k/n$ via EM steps, and it requires choosing the prior distribution for the signal $\mathbf{x}^*$. For our simulations we used the freely available GAMP package\footnote{For PR-GAMP we used the code available from \url{https://sourceforge.net/projects/gampmatlab/}} that does automatic parameter tuning, using the Gauss-Bernoulli prior. HWF is a much simpler algorithm, as it is just vanilla gradient descent applied to the unregularized empirical risk with Hadamard parametrization. HWF does not rely on algorithmic principles to promote convergence to good solutions, and it is empirically seen to adapt to the sparsity level $k$. We leave it to future work to give a full theoretical account on the convergence guarantees of HWF and to consider more refined and fine-tuned formulations of HWF that can combine algorithmic principles typically used in the literature on sparsity (cf.\ Section \ref{section:conclusion}, Conclusion).
\end{remark}
In experiments where we do not fix $x^*_{max}$, the true signal vector $\mathbf{x}^*\in \R^{1000}$ was obtained by sampling $\mathbf{x}^*\sim \mathcal{N}(0,\mathbf{I}_{1000})$, setting $(1000-k)$ random entries of $\mathbf{x}^*$ to $0$ and normalizing $\|\mathbf{x}^*\|_2=1$. Otherwise, $\mathbf{x}^*$ is generated as described in Section \ref{section:supportrecovery}. We obtain $m$ noiseless measurements $Y_j = (\mathbf{A}_j^T\mathbf{x}^*)^2$ with $\mathbf{A}_j\sim\mathcal{N}(0,\mathbf{I}_{1000})$ i.i.d..

For the parameters of SPARTA and SWF, we found the values suggested in the original papers to work best in our simulations and used these in all experiments. For HWF we found that a constant step size $\eta = 0.1$ works well (similar to WF \citep{MWCC18}). For the other parameters, any small values work well without much difference and we set $\alpha = 0.001$, $\kappa = 0.05$ and allow $\bar{b}=50$ restarts. We run all algorithms for a maximum of $\bar{t}=100,000$ iterations or until $F(\mathbf{X}^t)\le 10^{-7}$, and declare it a success if the relative error 
\begin{equation*}
\frac{\operatorname{dist}(\mathbf{X}^{\bar{t}}, \mathbf{x}^*)}{\|\mathbf{x}^*\|_2}
\end{equation*}
is less than $0.01$. We evaluate the empirical success rate obtained from 100 independent Monte Carlo trials.
In all experiments, SPARTA, SWF and SPARTA-support were run with oracle knowledge of the true signal sparsity $k$, which is not needed for HWF. 

In the first experiment, we fix the sparsity to $k=20$ and vary $m$ from $100$ to $1000$. Figure \ref{fig:k20m500}  (left) shows that HWF is able to reconstruct the signal reliably (with $95\%$ success rate) from $m=400$ measurements, which is slightly better than PR-GAMP ($m=500$), while SPARTA and SWF both require almost twice as many observations ($m=700$).
Next, we fix $m=500$ and vary the sparsity level $k$. Figure \ref{fig:k20m500} (right) shows that HWF achieves a reconstruction rate of $95\%$ for signals with up to $35$ non-zero entries, while the PR-GAMP achieves this success rate only for signals with up to $25$ non-zero entries. PR-GAMP achieves slightly higher success rates than HWF for sparsity levels where neither algorithm is able to reliably reconstruct the signal.

\begin{figure}[ht]
\begin{center}
\centerline{\includegraphics[width=0.66\columnwidth]{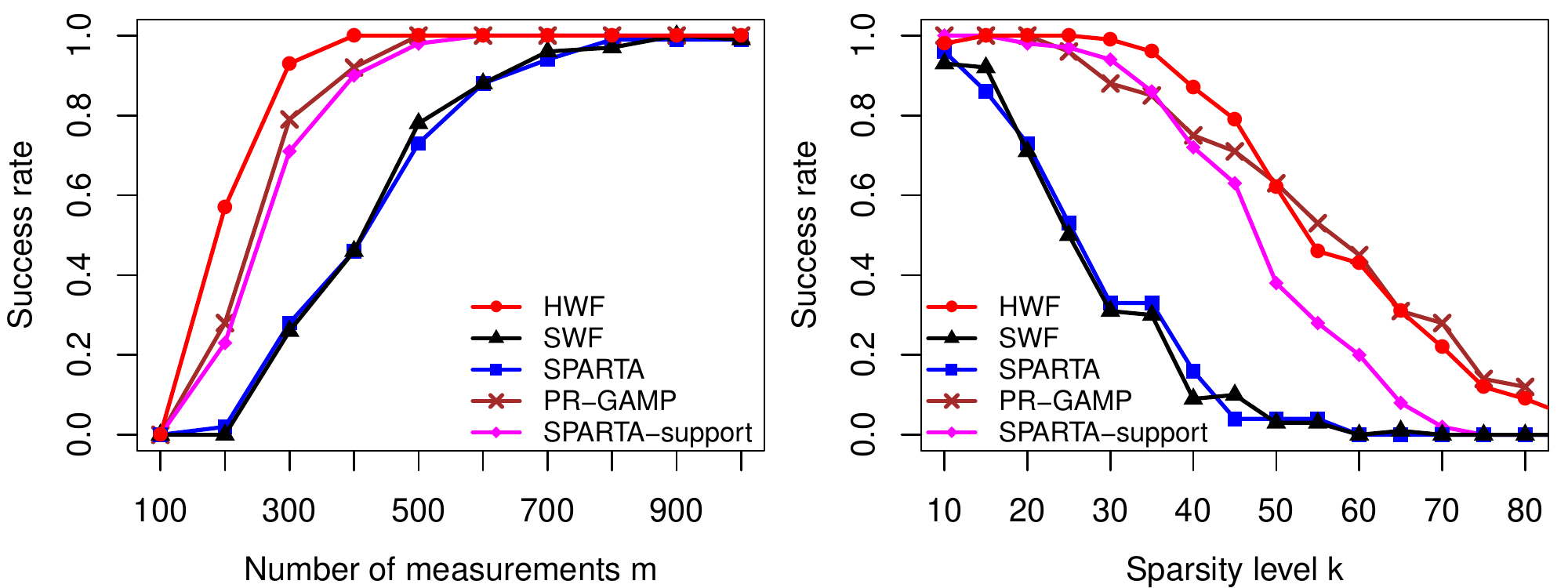}}
\caption{Empirical success rate for $n=1000$ fixed against number of measurements $m$ with sparsity level $k=20$ fixed (left) and against sparsity level $k$ with $m=500$ measurements (right).}
\label{fig:k20m500}
\end{center}
\vskip -0.3in
\end{figure}
In the previous section, we discussed the superior support recovery performance of our method as $x^*_{max}$ increases. The next experiment examines whether this effect also translates into better reconstruction performance. To this end, we consider signals generated as in the experiments in Section \ref{section:supportrecovery}, fix $m=500$ and vary $k\in[10,80]$. Figure \ref{fig:xmax} shows that, even when the signal only takes values $x^*_i\in \big\{-\frac{1}{\sqrt{k}}, 0, \frac{1}{\sqrt{k}}\big\}$, HWF achieves higher success rates than SPARTA and SWF. SPARTA-support is comparable to them, as also the support recovery performance is similar for this $\mathbf{x}^*$, and SPARTA-support subsequently applies the same steps as SPARTA. As $x^*_{max}$ increases, the reconstruction performance of our methods improves, with HWF maintaining a higher success rate than SPARTA-support. As before, PR-GAMP achieves a $95\%$ success rate up to slightly lower sparsity levels than HWF. PR-GAMP maintains success rates comparable to HWF as $x^*_{max}$ increases, which might explain the linear sample complexity observed in \citep{SR15} in some regimes for Gaussian signals. The maximum component of a Gaussian vector scales (in expectation) like $\sqrt{\log k}/\sqrt{k}$, which is, if $k$ is not very large, noticeably larger than $1/\sqrt{k}$. Comparing the right plot of Figure \ref{fig:k20m500} and the left plot of Figure \ref{fig:xmax}, we see that PR-GAMP achieves higher success rates for Gaussian signals than for the signal with $x^*_{max} = 1/\sqrt{k}$.

\begin{figure}[ht]
\begin{center}
\centerline{\includegraphics[width=\columnwidth]{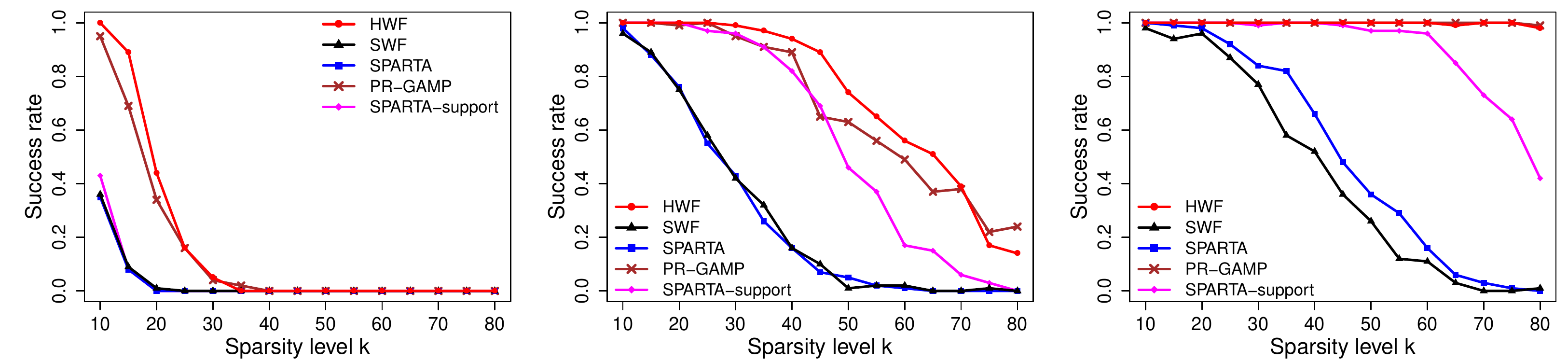}}
\caption{Empirical success rate against sparsity level $k$ with $n=1000, m=500$ fixed. From left to right: $(i)$ $x^*_{max} = 1/\sqrt{k}$, $(ii)$ $x^*_{max} = k^{-0.25}$ and $(iii)$ $x^*_{max} = 0.7$.} 
\label{fig:xmax}
\end{center}
\vskip -0.3in
\end{figure}
Next, we examine how the sample complexity of HWF scales with the signal sparsity $k$. The success rate vs signal sparsity $k$ and number of measurements $m$ is shown in Figure \ref{fig:heatmap}, which suggests that the sample complexity scales as $\O(k(x^*_{max})^{-2}\log\frac{n}{k})$, where we obtain $x^*_{max}$ as the average maximum coordinate of $100,000$ Gaussian $k$-sparse signals. We note that this scaling appears almost linear.

\begin{figure}[ht]
\begin{center}
\centerline{\includegraphics[width=0.6\columnwidth]{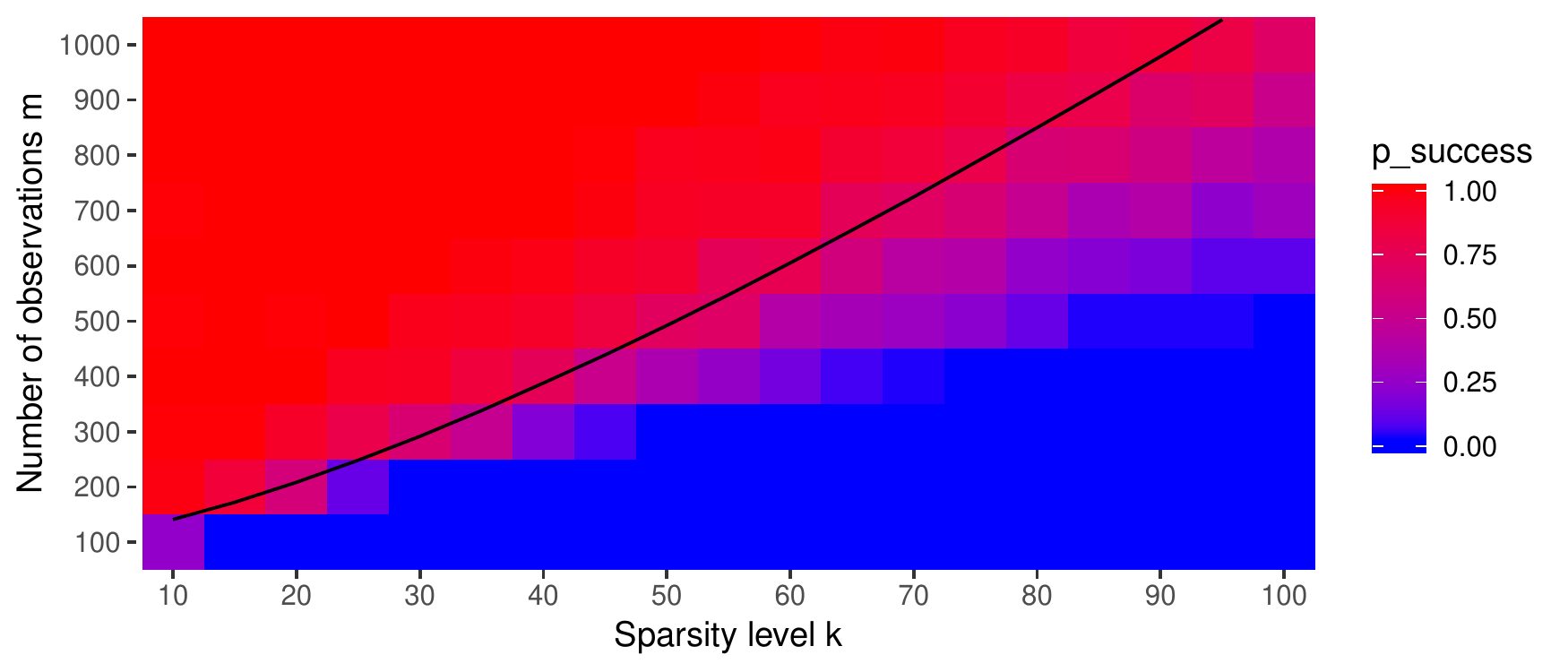}}
\caption{Empirical success rate (red: high, blue: low) of HWF against sparsity level $k$ and number of observations $m$, with $n=1000$ fixed. Black line: $m=\frac{1}{3}k(x^*_{max})^{-2}\log \frac{n}{k}$.}
\label{fig:heatmap}
\end{center}
\vskip -0.3in
\end{figure}
One of the parameters in Algorithm \ref{alg:wfrmr} is the number of restarts $\bar{b}$. Increasing $\bar{b}$ also increases the probability of HWF finding the true signal, but this comes at the cost of an increase in computational time. For the next experiment, we run HWF in the same setting as the first two experiments and vary the number of allowed restarts $\bar{b}$ from $1$ to $100$. Figure \ref{fig:numtrial} shows that increasing the number of allowed restarts indeed increases the probability of successful reconstruction, where the success rate barely increases further as we increase the number of restarts $\bar{b}$ beyond $50$.

\begin{figure}[ht]
\begin{center}
\centerline{\includegraphics[width=0.66\columnwidth]{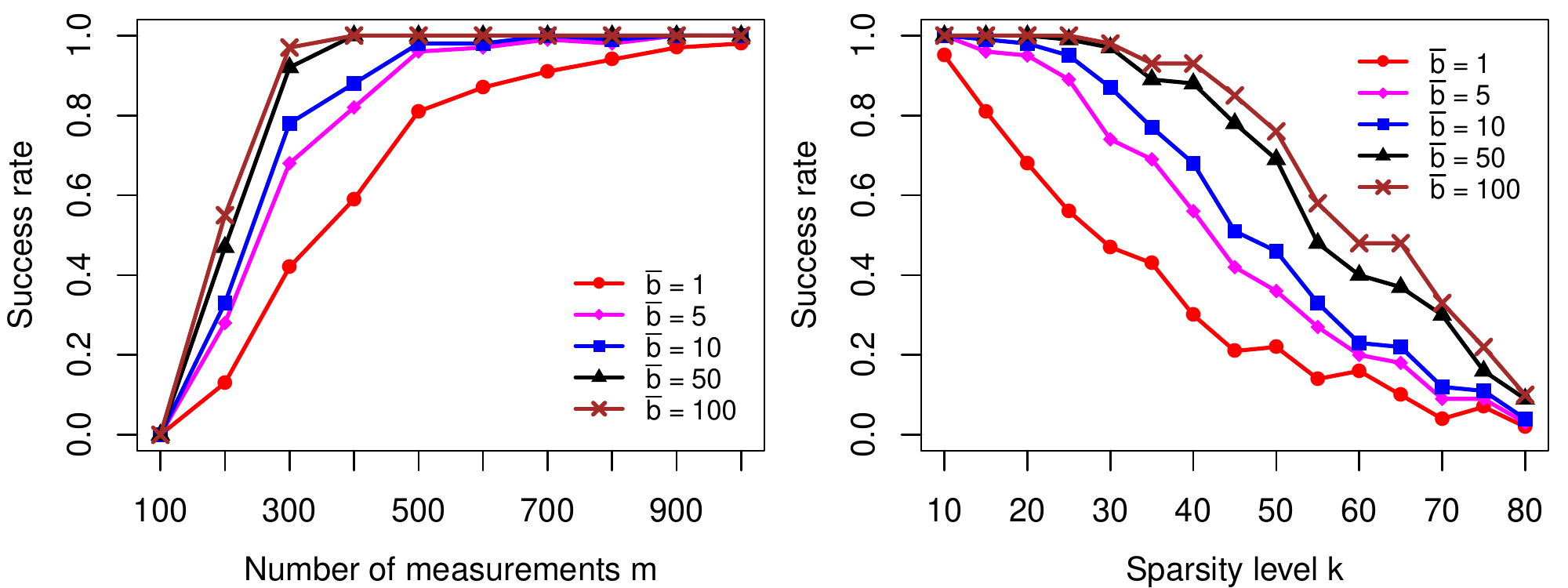}}
\caption{Empirical success rate for $n=1000$ fixed against number of measurements $m$ with sparsity level $k=20$ fixed (left), and against sparsity level $k$ with $m=500$ measurements (right), for varying number of restarts $\bar{b}\in [1,100]$.} 
\label{fig:numtrial}
\end{center}
\vskip -0.4in
\end{figure}
Next, we examine the convergence behavior of HWF. We also test HWF in the complex-valued setting, where we generate vectors $\mathbf{x}^*, \mathbf{A}_j\sim\mathcal{N}(0,\frac{1}{2}\mathbf{I}_{1000}) + i\mathcal{N}(0, \frac{1}{2}\mathbf{I}_{1000})$, set $990$ random entries of $\mathbf{x}^*$ to zero, normalize $\|\mathbf{x}^*\|_2=1$ and generate $m=500$ measurements $Y_j = |\mathbf{A}_j^\mathrm{H}\x^*|^2$. Figure \ref{fig:conv_complex} shows that HWF is also able to reconstruct complex signals. While HWF converges faster in the real case, both cases exhibit sublinear convergence after a short "warm-up" period. 
This can be explained by our parametrization. Consider the gradient $\nabla_{\mathbf{u}} F(\mathbf{U}^t,\mathbf{V}^t) = 2\nabla F(\mathbf{X}^t)\odot \mathbf{U}^t$: as the initialization size $\alpha$ is small, the gradient is small in the beginning due to the term $\mathbf{U}^t$. As $\mathbf{X}^t$ approaches $\mathbf{x}^*$, the term $\nabla F(\mathbf{X}^t)$ converges to zero, which leads to linear convergence with a constant stepsize in the case of WF \citep{MWCC18}. With our parametrization, $\nabla_{\mathbf{u}} F$ (or $\nabla_{\mathbf{v}} F$) converges to zero faster than $\nabla_{\mathbf{x}} F$, as we typically have $U^t_i\rightarrow 0$ or $V^t_i\rightarrow 0$ (or both, if $x^*_i=0$), leading to sublinear convergence.

\begin{figure}[ht]
\begin{center}
\centerline{\includegraphics[width=0.7\columnwidth]{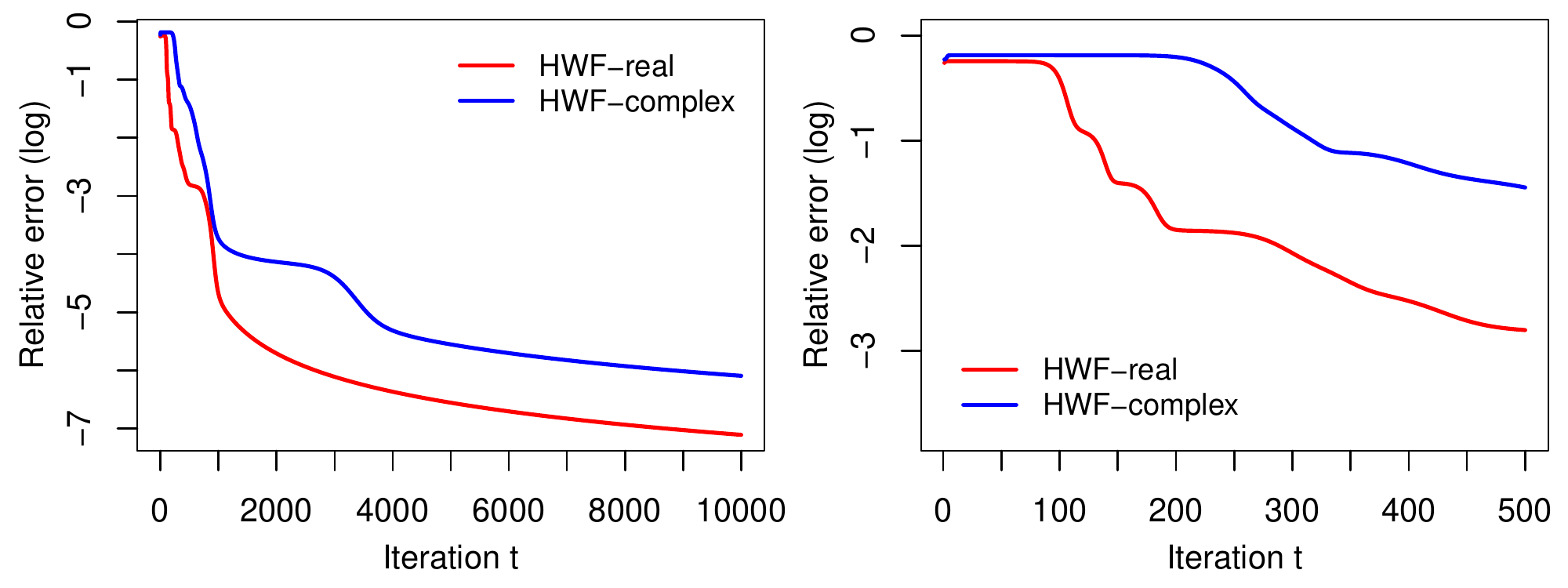}}
\caption{Relative error (log-scale) of HWF for real/complex signals with $n = 1000$, $m=500$ and $k=10$ for 10,000 iterations (left) and zoom-in to 500 iterations (right).} 
\label{fig:conv_complex}
\end{center}
\vskip -0.3in
\end{figure}
Finally, we present a numerical experiment that considers random initialization. In particular, we initialize $U^0_i, V^0_i$ to small Gaussian noise $\gauss (0, 0.01^2)$ for all $i = 1,\dots, n$. In general, we find that more samples are required for HWF to successfully reconstruct the signal $\x^*$ starting from a random initialization. Figure \ref{fig:randini} shows that even in a setting with $n=1000$, $m=700$ and $k=10$, where HWF with random initialization does converge to the signal $\x^*$, the $\ell_2$ error $\operatorname{dist}(\mathbf{X}^t, \x^*)$ only decreases after an initial plateau, leading to slower convergence. This is in line with the intuition provided in the appendix, namely that $(i)$ the signal can still be recovered as coordinates on the support of $\x^*$ increase at a faster rate than coordinates not on the support, while $(ii)$ all coordinates only change at a very slow rate initially, because the inner product $(\mathbf{X}^0)^T\x^*$ is closer to zero with random initialization compared to our proposed initialization (\ref{eq:initialization}), which leads to the initial plateau; see the appendix for more details. 
\begin{figure}[ht]
\begin{center}
\centerline{\includegraphics[width=0.4\columnwidth]{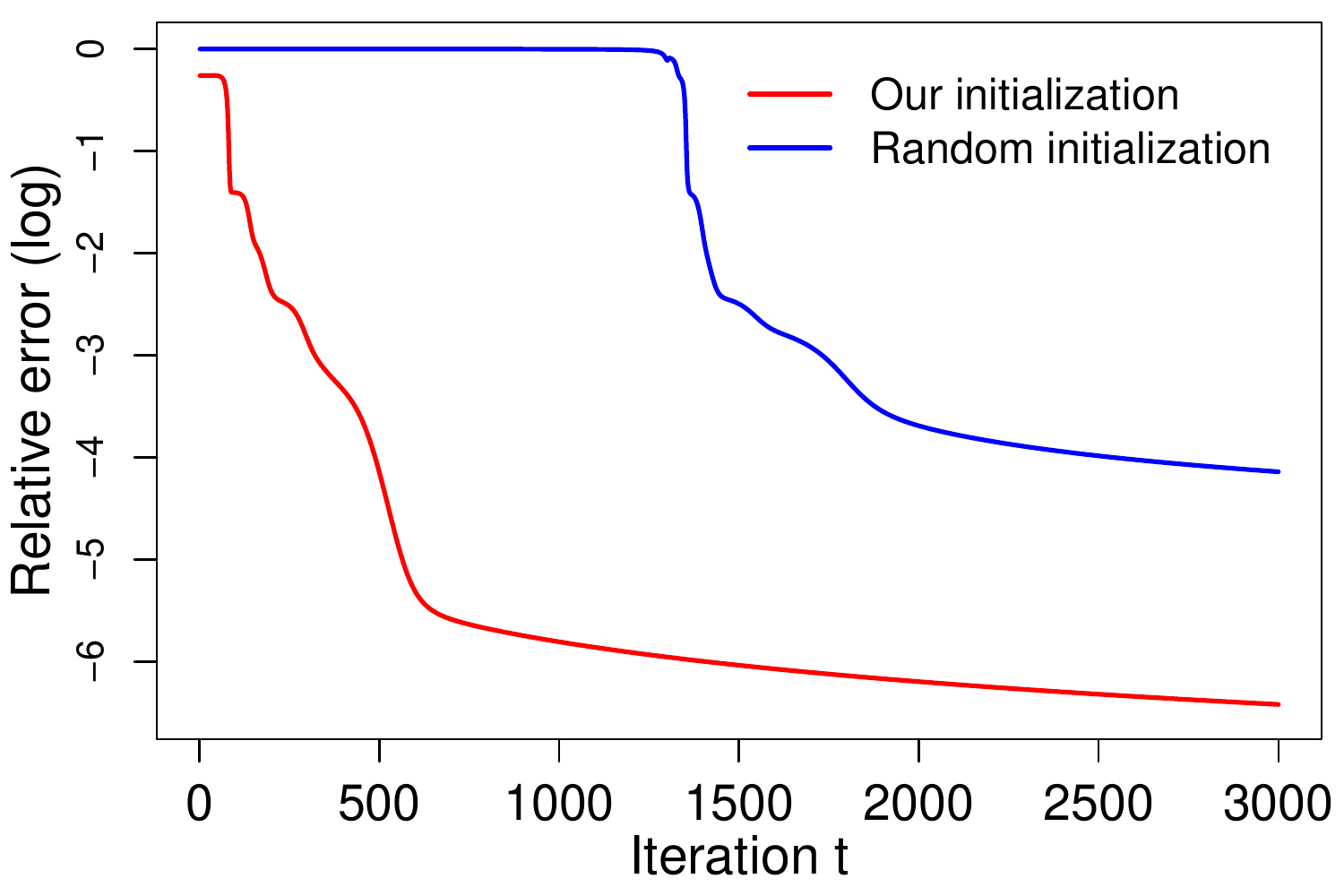}}
\caption{Relative error (log-scale) of HWF with random initialization (blue) and our proposed initialization (\ref{eq:initialization}) (red) with $n=1000$, $m=700$ and $k=10$.} 
\label{fig:randini}
\end{center}
\vskip -0.3in
\end{figure}

\section{Conclusion}
\label{section:conclusion}
In this paper, we proposed HWF, which is a simple algorithm for sparse phase retrieval. We proved that one step of HWF can be used as a support recovery tool, which, combined with existing algorithms such as SPARTA, yields a computationally fast algorithm with improved sample complexity, which reads $\O(k\log n)$ if the signal contains at least one large component and $k\ge \O(\log^2n)$. We have shown in numerical experiments that the sample complexity of HWF is lower than that of existing gradient based methods such as SPARTA and SWF, and comparable to PR-GAMP, which has been empirically shown to achieve linear sample complexity for Gaussian signals in some regimes \citep{SR15}. While HWF does not require knowledge of the signal sparsity $k$, thresholding steps or any added regularization terms, this simplicity seems to come at the price of sublinear convergence and thus increased computational cost.
We leave it to future work to investigate whether algorithmic principles such as the increasing step-size scheme considered in \citep{VKR19} or thresholding steps previously considered in the literature on sparse phase retrieval (e.g.\ \citep{CLM16, WZGAC18, ZWGC18}) can be used to accelerate the convergence speed of HWF or to further improve its sample complexity beyond the level of the empirical results observed for PR-GAMP, which already relies on a combination of many such algorithmic principles (damping, normalization, EM steps). Compared to PR-GAMP, the simplicity of HWF makes the algorithm potentially more amenable to a rigorous theoretical investigation that can support the high-level analysis presented in our work.

\bibliography{references}
\bibliographystyle{plainnat}

\clearpage
\appendix
\section{Understanding the Dynamics of Hadamard Wirtinger Flow}
\label{appendix:understanding}
As discussed in Section \ref{section:hwf}, the Hadamard parametrization has previously been applied to problems such as sparse recovery \citep{H17, VKR19, ZYH19} and matrix factorization \citep{GWBNS17, LMZ18, ACHL19}, where it turns the additive updates of gradient descent into multiplicative updates. The combination of multiplicative updates and a small initialization was shown to lead to sparsity in the aforementioned problems, under the assumption of the restricted isometry property (RIP).

The problem of sparse phase retrieval that we consider is known to satisfy the RIP property, cf.\ \citep{VX16} for instance, and a similar explanation on why on-support variables and off-support variables can be made to grow at different speeds also holds in our setting. We now provide the main intuition behind the convergence properties of HWF by considering the evolution of the algorithm at the population level, i.e.\ in the case when $m = \infty$. While a rigorous convergence investigation of HWF is outside the scope of the present work, the analysis that we now provide is instrumental to construct a good initialization for Algorithm \ref{alg:wfrmr}.

Consider the simplified setting where $\mathbf{x}^*$ is non-negative, i.e.\ $x^*_i\ge 0$ for all $i$. We can set $\mathbf{v}=\mathbf{0}$ in the parametrization, so that $\mathbf{x}=\mathbf{u}^2$. Further, assume that we have access to the population risk $f(\mathbf{x}) := \E[\ell(\mathbf{x}, \mathbf{Z})]$ (in other words, $m=\infty$), where $\mathbf{Z}=(Y,\mathbf{A})$ is defined by $Y = (\mathbf{A}^T\x^*)^2$. Its gradient can be computed as
\begin{equation}\label{eq:pop_grad}
\nabla f(\mathbf{x}) = \big(3\|\mathbf{x}\|_2^2-1\big)\mathbf{x} - 2\big(\mathbf{x}^T\mathbf{x}^*\big)\mathbf{x}^*.
\end{equation} 
Under these two assumptions, first consider the initialization $\mathbf{x}^0=\alpha^2\Eins_n$ for some small constant $\alpha>0$. We can directly track the evolution of the estimates $\mathbf{x}^t$ (note that we use lowercase letters, since with $m=\infty$ the sequence is not random anymore) generated by Algorithm \ref{alg:wfrmr} via 
\begin{equation*}
x^{t+1}_i = x^t_i \big(1-2\eta \big[\big(3\|\mathbf{x}^t\|_2^2-1\big)x^t_i-2 \big((\mathbf{x}^t)^T\mathbf{x}^*\big)x^*_i\big]\big)^2.
\end{equation*}
This suggests that the evolution of $\mathbf{x}^t$ can be divided into two phases: if $\|\mathbf{x}^t\|_2^2 < \frac{1}{3}$, all coordinates grow ($x^{t+1}_i > x^t_i$), while coordinates $i\in \mathcal{S}$ on the support do so at a faster rate. If $\|\mathbf{x}^t\|_2^2>\frac{1}{3}$, coordinates $i\notin \mathcal{S}$ decrease ($x^{t+1}_i<x^t_i$), while coordinates on the support increase if the product of the signal component $x^*_i$ and the inner product $(\mathbf{x}^t)^T\mathbf{x}^*$ is larger than the term $(3\|\mathbf{x}^t\|_2^2-1)x^t_i$. 

If we choose $\alpha>0$ small enough, we expect $x^t_j$ to still be small (e.g.\ $<1/n$) for $j\notin \mathcal{S}$ when $\|\mathbf{x}^t\|_2^2\ge\frac{1}{3}$ first occurs, as $x^t_i$ grows at a faster rate than $x^t_j$ for $i\in \mathcal{S}$. Since $x^t_j$ decreases for $j\notin \mathcal{S}$ when $\|\mathbf{x}^t\|_2^2\ge\frac{1}{3}$, we expect $x^t_j$ to stay small throughout the algorithm for $j\notin \mathcal{S}$.

The smaller the step size $\eta$ is, the more iterations are needed for the algorithm to converge. On the other hand, $\eta$ cannot be too large; to illustrate this, consider the simplest case $n=1$. The (scalar) gradient update becomes $x^{t+1} = x^t(1-6\eta[(x^t)^3-x^t)])$, and $x^t$ diverges if $\eta$ is too large. We found a constant step size $\eta=0.1$ to work well in our simulations.

This recursion has three types of fixed points: $\mathbf{x}^{(1)}=\mathbf{0}$, any $\mathbf{x}^{(2)}$ satisfying $\|\mathbf{x}^{(2)}\|_2^2=\frac{1}{3}$ and $(\mathbf{x}^{(2)})^T\mathbf{x}^*=0$, and $\mathbf{x}^{(3)} = \pm \mathbf{x}^*$. The first fixed point $\mathbf{x}^{(1)}$ is repelling, as all coordinates grow if $\|\mathbf{x}^t\|_2^2 < \frac{1}{3}$. Similarly, the second fixed point $\mathbf{x}^{(2)}$ is repelling as $x^t_i$ grows at a faster rate than $x^t_j$ for $i\in \mathcal{S}, j \notin \mathcal{S}$. This leaves only $\mathbf{x}^{(3)}$, which is an attracting fixed point of the recursion. Thus, we expect Algorithm \ref{alg:wfrmr} to converge to $\mathbf{x}^*$ if $m$ is sufficiently large.

Guided by this intuition, we aim to construct an initialization $\mathbf{X}^0$ with $(\mathbf{X}^0)^T\mathbf{x}^*$ large (more precisely, we will have $|(\mathbf{X}^0)^T\mathbf{x}^*| \ge \frac{1}{4}x^*_{max}$), while at the same time $\|\mathbf{X}^0\|_2^2$ should not be too large (e.g.\ fixed to $\|\mathbf{X}^0\|_2^2 = \frac{1}{3}\|\x^*\|_2$; note that any other constant would also work, and that we use the estimate $\hat{\theta} = (\frac{1}{m}\sum_{j=1}^mY_j)^{1/2}$ of the signal size $\|\x^*\|_2$, see e.g. \citep{CLS15, WGE17}). In order to obtain such an initialization, it suffices to find a coordinate $i\in [n]$ with $|x^*_i|\ge \frac{1}{2}x^*_{max}$. Then, we can set $X^0_i = \hat{\theta}/\sqrt{3}$ and $X^0_j=0$ for all $j\neq i$. Note that such an initialization is not necessary, and even with a random initialization (e.g.\ $U^0_i, V^0_i$ set to small random noise for all $i=1,\dots , n$), the above intuition that coordinates $i\in \mathcal{S}$ on the support grow at a faster rate than coordinates $i\notin \mathcal{S}$ not on the support, continues to hold. However, the initial inner product $(\mathbf{X}^0)^T\x^*$ is closer to zero with random initialization compared to our proposed initialization, which leads to the population gradient $\nabla f(\mathbf{X}^t)_i$ initially being close to zero for all $i=1,\dots ,n$, and therefore slow convergence.

Define the random variables $R_i = \frac{1}{m}\sum_{j=1}^mY_jA^2_{ji}$ for $i = 1,\dots, n$. These quantities were also used in \citep{WZGAC18} for support recovery, as one can compute $\E[R_i] = \|\mathbf{x}^*\|_2^2 + 2x_i^2$ using the assumption $\mathbf{A}_j\sim \mathcal{N}(0,\mathbf{I}_n)$ i.i.d.. Hence, if the number of measurements $m$ is large, the random variables $\{R_i\}_{i=1}^n$ will concentrate around their means, separating them for $i\in \mathcal{S}$ and $i\notin \mathcal{S}$. This intuition suggests the initialization proposed in Section \ref{section:hwf}. 

\section{Proof of Lemma \ref{lemma:suprec}}
\label{appendix:proof_of_lemma_1}
In the following, we assume, without loss of generality, that $\|\mathbf{x}^*\|_2=1$; this assumption is made purely for notational simplicity, since we then have $x^*_{max} = \underset{i}{\max} \frac{|x^*_i|}{\|\mathbf{x}^*\|_2} = \underset{i}{\max}|x^*_i| $ and $x^*_{min} = \underset{i: x^*_i\neq 0}{\min}\frac{|x^*_i|}{\|\mathbf{x}^*\|_2} = \underset{i: x^*_i\neq 0}{\min} |x^*_i|$. If $\|\mathbf{x}^*\|_2\neq 1$ is unknown, then we only need to replace $x^*_{max}$ and $x^*_{min}$ with $\underset{i}{\max}|x^*_i|=x^*_{max}\|\mathbf{x}^*\|_2$ and $\underset{i: x^*_i\neq 0}{\min} |x^*_i|= x^*_{min}\|\mathbf{x}^*\|_2$ respectively in the following proof. Further, note that knowledge of $\|\mathbf{x}^*\|_2$ is not required for HWF.

The proof of Lemma \ref{lemma:suprec} relies on the following result, which is a combination of Theorems 3.6 and 3.7 of \citep{CL06}.
\begin{theorem}\citep{CL06}\label{thm:ref}
Let $X_i$ be independent random variables satisfying $|X_i| \le M$ for all $i\in [n]$. Let $X = \sum_{i=1}^nX_i$ and $\|X\|=\sqrt{\sum_{i=1}^n\E[X_i^2]}$. Then, we have
\begin{equation*}
\mathbb{P}[|X-\E[X]| > \lambda]\le 2\exp\biggl(- \frac{\lambda^2}{2(\|X\|^2 + M\lambda/3)} \biggr).
\end{equation*}
\end{theorem}

\textbf{Proof of the first claim.}

We first show that by choosing the largest instance in $\{\frac{1}{m}\sum_{j=1}^mY_jA^2_{ji}\}_{i=1}^n$, we obtain an index $i$ with $|x^*_i|\ge \frac{x^*_{max}}{2}$ with high probability. 
Recall that we write $R_i = \frac{1}{m}\sum_{j=1}^mY_jA^2_{ji}$. We can compute 
\begin{align*}
\E[R_i] &= \E[(\mathbf{A}_1^T\mathbf{x}^*)^2A_{1i}^2] \\
&= \E[A_{1i}^4(x^*_i)^2 + (\mathbf{A}_{1,-i}^T\mathbf{x}^*_{-i})^2A_{1i}^2] \\
&= 3(x^*_i)^2 + \|\mathbf{x}^*_{-i}\|_2^2 \\
&=\|\mathbf{x}^*\|_2^2 + 2(x_i^*)^2,
\end{align*}
where we denote by $\mathbf{x}_{-i}\in \R^{n-1}$ the vector obtained by deleting the $i$-th entry from $\mathbf{x}\in\R^n$ and use the fact that $A_{ji}\sim \mathcal{N}(0,1)$ i.i.d.\ and hence $\mathbf{A}_{j,-i}^T\mathbf{x}^*_{-i}\sim \mathcal{N}(0, \|\mathbf{x}^*_{-i}\|_2^2)$, as $\mathbf{x}^*\in\R^n$ is a fixed vector independent of the measurement vectors $\{\A_j\}_{j=1}^m$.

Let $I_{max} = \operatorname{argmax}_i R_i$. By definition, $R_{I_{max}}\ge R_i$ holds for all $i\in[n]$. If we can show $|R_i-\E[R_i]|\le \frac{3}{4}(x^*_{max})^2$ for all $i\in [n]$, then this would imply
\begin{align*}
\|\mathbf{x}^*\|_2^2 + 2(x^*_{I_{max}})^2 &= \E[R_{I_{max}}] \\
&= \E[R_{i}] + (R_i - \E[R_{i}]) + (\E[R_{I_{max}}]-R_{I_{max}}) +(R_{I_{max}} - R_i) \\
&\ge \E[R_i] - 2\max_j |R_j-\E[R_j]| \\
&\ge \|\mathbf{x}^*\|_2^2 + 2(x^*_i)^2 - \frac{3}{2}(x^*_{max})^2,
\end{align*}
for any $i\in[n]$.
In particular, if we choose $i=\operatorname{argmax}_j |x^*_j|$, this implies $|x^*_{I_{max}}|\ge \frac{1}{2}x^*_{max}$, which concludes the proof of the first claim.

In order to show $|R_i-\E[R_i]|\le \frac{3}{4}(x^*_{max})^2$, we use the following truncation argument: for any $i\in [n]$, we write
\begin{align*}
R_i = \frac{1}{m}\sum_{j=1}^m Y_j A^2_{ji} = \frac{1}{m}\sum_{j=1}^m (\mathbf{A}_j^T\mathbf{x}^*)^2 A^2_{ji} = \frac{1}{m}\sum_{j=1}^m (Z_{1,j} + Z_{2,j}),
\end{align*}
where $Z_{1,j} = (\mathbf{A}_j^T\mathbf{x}^*)^2 A^2_{ji} \cdot \Eins(\max\{|\mathbf{A}_j^T\mathbf{x}^*|, |A_{ji}|\} < \sqrt{44\log n})$ and $Z_{2,j} = (\mathbf{A}_j^T\mathbf{x}^*)^2 A^2_{ji} - Z_{1,j}$. Since $Z_{1,j}$ is bounded, we can apply Theorem \ref{thm:ref}. To this end, compute the second moment
\begin{equation*}
\sum_{j=1}^m\frac{1}{m^2}\E[Z_{1,j}^2] \le \sum_{j=1}^m\frac{1}{m^2}\E\bigl[(\mathbf{A}_j^T\mathbf{x}^*)^4 A^4_{ji}\bigr] \le \sum_{j=1}^m\frac{1}{m^2}\sqrt{\E\bigl[(\mathbf{A}_j^T\mathbf{x}^*)^8\bigr]\E\bigl[ A^8_{ji}\bigr]} \le \frac{105}{m},
\end{equation*}
where we used the Cauchy-Schwarz inequality and the fact that $\mathbf{A}_j^T\mathbf{x}^*\sim \gauss(0,1)$. With this, we have
\begin{equation*}
\P\Biggl[\bigg|\frac{1}{m}\sum_{j=1}^mZ_{1,j} - \E[Z_{1,j}]\bigg| > \frac{3}{8}(x^*_{max})^2\Biggr]\le 2\exp\Biggl(- \frac{\frac{9}{64}(x^*_{max})^4}{2(\frac{105}{m} + \frac{44^2\log^2n}{m} \cdot \frac{3}{8}(x^*_{max})^2/3)}\Biggr) \le \O(n^{-11})
\end{equation*}
since $m\ge \O(\max\{k\log n,\, \log^3 n\}(x^*_{max})^{-2})$.

For the second term $Z_{2,j}$, we can use the Chebyshev inequality: we have
\begin{align*}
\operatorname{Var}\Biggl(\frac{1}{m}\sum_{j=1}^mZ_{2,j}\Biggr) &\le \frac{1}{m}\E\Bigl[(\mathbf{A}_1^T\mathbf{x}^*)^4 A^4_{1i} \cdot \Eins\Bigl(\max\{|\mathbf{A}_1^T\mathbf{x}^*|, |A_{1i}|\} > \sqrt{44\log n}\Bigr)\Bigr] \\
&\le \frac{1}{m}\sqrt{\E\bigl[(\mathbf{A}_1^T\mathbf{x}^*)^8 A^8_{1i}\bigr] \cdot \P\Bigl[\max\{|\mathbf{A}_1^T\mathbf{x}^*|, |A_{1i}|\} > \sqrt{44\log n}\Bigr]}\\
&\le \frac{45\sqrt{1001}}{m} \cdot 2n^{-11},
\end{align*}
and hence, by the Chebyshev inequality,
\begin{equation*}
\P\Biggl[\bigg|\frac{1}{m}\sum_{j=1}^mZ_{2,j} - \E[Z_{2,j}]\bigg| > \frac{3}{8}(x^*_{max})^2\Biggr]\le \frac{\frac{45\sqrt{1001}}{m}\cdot 2n^{-11}}{\frac{9}{64}(x^*_{max})^4} \le \O(n^{-11}).
\end{equation*}
Put together, this implies that
\begin{equation*}
\mathbb{P}\biggl[|R_i-\E[R_i]|>\frac{3}{4}(x^*_{max})^2\biggr] \le \O(n^{-11}).
\end{equation*}
Taking the union bound over all $i\in [n]$ implies that $|R_i-E[R_i]|\le \frac{3}{4}(x^*_{max})^2$ holds for all $i\in [n]$ with probability at least $1-O(n^{-10})$. This concludes the proof of the first claim of Lemma \ref{lemma:suprec}.

\textbf{Proof of the second claim.}

First, note that the gradient of the empirical risk $F(\mathbf{x})$ is given by 
\begin{equation*}
\nabla F(\mathbf{x}) = \frac{1}{m}\sum_{j=1}^m \bigl((\mathbf{A}_j^T\mathbf{x})^2 - (\mathbf{A}_j^T\mathbf{x}^*)^2\bigr)(\mathbf{A}_j^T\mathbf{x})\mathbf{A}_j.
\end{equation*}
By the dominated convergence theorem, the gradient of the population risk $f(\mathbf{x})$ can then be computed as
\begin{align*}
\nabla f(\mathbf{x}) = \E[\nabla F(\mathbf{x})] &= \E\bigl[\bigl((\mathbf{A}_1^T\mathbf{x})^2 - (\mathbf{A}_1^T\mathbf{x}^*)^2\bigr)(\mathbf{A}_1^T\mathbf{x})\mathbf{A}_1\bigr]\\
&=\big(3\|\mathbf{x}\|_2^2-1\big)\mathbf{x}-2\big(\mathbf{x}^T\mathbf{x}^*\big)\mathbf{x}^*
\end{align*}
for any fixed vector $\mathbf{x}\in\R^n$.
Further, we have the initialization
\begin{align*}
U^0_i &= \begin{cases}\Bigl(\frac{\hat{\theta}}{\sqrt{3}} + \alpha^2\Bigr)^{\frac{1}{2}} \qquad & i=I_{max} \\ \alpha & i\neq I_{max} \end{cases} \\
V^0_i &= \alpha
\end{align*}
which leads to
\begin{equation*}
X^{0}_i = \begin{cases} \frac{\hat{\theta}}{\sqrt{3}} \qquad &i=I_{max} \\ 0 & i\neq I_{max} \end{cases}
\end{equation*}
Hence, we have 
\begin{equation*}
\nabla f(\mathbf{X}^0)_i = \bigl(\hat{\theta}^2 - 1\bigr)X^0_i -\frac{2\hat{\theta}}{\sqrt{3}} x^*_{I_{max}}x^*_i. 
\end{equation*}
In particular, we have $\nabla f(\mathbf{X}^0)_j = 0$ for $j\notin \mathcal{S}$.
Using standard concetration bounds for sub-exponential random variables (see e.g.\ Prop. 5.16 of \citep{V12}), we can bound with probability $1 - \O(n^{-10})$ (recall that we have assumed $\|\x^*\|_2 = 1$ for notational simplicity), 
\begin{equation*}
\big|\hat{\theta}^2 - 1\big| = \bigg|\frac{1}{m}\sum_{j=1}^m (\A_j^T\x^*)^2 - 1 \bigg| \le 9\sqrt{\frac{\log n}{m}}.
\end{equation*}
This bound implies $2\hat{\theta} \ge \sqrt{3}$, where we used that $m\ge \O(\max\{k\log n,\, \log^3n\}(x^*_{max})^{-2})$.

For the second claim we need to show that $|X^1_i|>|X^1_j|$ holds whenever $i\in \mathcal{S}$ and $j\notin \mathcal{S}$. We can assume without loss of generality that $x^*_{I_{max}}>0$. First, consider the case $i\neq I_{max}$. Since $|X^1_i| = |(U^1_i)^2-(V^1_i)^2|$, it suffices to show, assuming $x^*_i>0$, that
\begin{align}
U^1_i &> \max\{U^1_j, V^1_j\} \label{eq:toshow1}\\
V^1_i &< \min\{U^1_j, V^1_j\} \label{eq:toshow2}
\end{align}
holds simultaneously. The case $x^*_i<0$ can be dealt with the same way, exchanging the roles of $U^1_i$ and $V^1_i$.
We can bound 
\begin{align*}
U^1_i &= \alpha \big(1-2\eta\nabla F(\mathbf{X}^0)_i\big) \nonumber\\
&\ge \alpha \big(1 - 2\eta\nabla f(\mathbf{X}^0)_i - 2\eta|\nabla F(\mathbf{X}^0)_i - \nabla f(\mathbf{X}^0)_i|\big),
\end{align*}
and 
\begin{equation*}
U^1_j \le \alpha \big(1 + 2\eta |\nabla F(\mathbf{X}^0)_j - \nabla f(\mathbf{X}^0)_j|\big).
\end{equation*}
We have shown above that (recall that $X^0_i=0$ for $i\neq I_{max}$)
\begin{equation*} 
-\nabla f(\mathbf{X}^0)_i = \frac{2\hat{\theta}}{\sqrt{3}}x^*_{I_{max}}x^*_i \ge \frac{1}{2}x^*_{max}x^*_{min},
\end{equation*}
since from the first part we know that $x^*_{I_{max}} \ge \frac{1}{2}x^*_{max}$ and we assumed $x^*_{i}>0$. Hence, if we can show 
\begin{equation}\label{eq:toshow3}
\max_i |\nabla F(\mathbf{X}^0)_i - \nabla f(\mathbf{X}^0)_i| \le \frac{1}{4}x^*_{max}x^*_{min},
\end{equation}
then $U^1_i\ge U^1_j$ follows. We also have 
\begin{equation*}
V^1_j \le \alpha \big(1 + 2\eta |\nabla F(\mathbf{X}^0)_j - \nabla f(\mathbf{X}^0)_j|\big),
\end{equation*}
which then implies $U^1_i\ge V^1_j$, completing the proof of (\ref{eq:toshow1}); (\ref{eq:toshow2}) can be shown the same way.

The case $i=I_{max}$ also follows from the bound (\ref{eq:toshow3}). Since $m\ge \O(k(x^*_{max})^{-2}\log n)$, we can bound 
\begin{align*}
|\nabla F(\mathbf{X}^0)_i| &\le |\nabla f(\mathbf{X}^0)_i| + |\nabla F(\mathbf{X}^0)_i - \nabla f(\mathbf{X}^0)_i| \\
&\le 9\sqrt{\frac{\log n}{m}} \frac{\hat{\theta}}{\sqrt{3}} + \frac{2\hat{\theta}}{\sqrt{3}} (x^*_{max})^2 + \frac{1}{2\sqrt{3}}x^*_{max}x^*_{min} \\
&\le 2,
\end{align*}
where we used that $x^*_{max} \le 1$. Since we assume $\eta\le 0.1$, we can bound 
\begin{equation*}
|2\eta\nabla F(\mathbf{X}^0)_i| \le 0.4,
\end{equation*}
which, since also $\alpha\le 0.1$, implies 
\begin{equation*}
U^1_i\ge \biggl(\frac{\hat{\theta}}{\sqrt{3}} + \alpha^2\biggr)^{\frac{1}{2}} (1 - 0.4) \ge 2 \max\{U^1_j, V^1_j, V^1_i\},
\end{equation*}
and hence 
\begin{equation*}
|X^1_i| = (U^1_i)^2 - (V^1_i)^2 \ge \max\left\{(U^1_j)^2, (V^1_j)^2\right\} \ge |X^1_j|.
\end{equation*}
What is left to show is (\ref{eq:toshow3}). Since $\mathbf{X}^0$ is not independent from $\{\A_j\}_{j=1}^m$, we cannot immediately apply the truncation argument from the proof of the first claim. Therefore, define the (deterministic) vectors $\mathbf{x}^{(l)}\in \R^n$ for $l=1,...,n$ by
\begin{equation*}
x^{(l)}_i = \begin{cases} \frac{1}{\sqrt{3}} \qquad &i=l \\ 0 &i\neq l\end{cases}
\end{equation*}
Now, we need to show that the empirical gradient
\begin{align*}
\nabla F(\mathbf{x}^{(l)})_i &= \frac{1}{m}\sum_{j=1}^m((\mathbf{A}_j^T\mathbf{x}^{(l)})^2 - (\mathbf{A}_j^T\mathbf{x}^*)^2)(\mathbf{A}_j^T\mathbf{x}^{(l)})A_{ji}
\end{align*}
is close to its expectation $\nabla f(\mathbf{x}^{(l)})_i$. Using the same truncation argument as in the proof of the first claim, we can show 
\begin{align*}
\mathbb{P}\biggl[\big|\nabla f(\mathbf{x}^{(l)})_i - \nabla F(\mathbf{x}^{(l)})_i\big| \ge \frac{1}{8}x^*_{max}x^*_{min}\biggr]
\le \O(n^{-12}),
\end{align*}
Taking the union bound over all $i$ and $l$ implies that
\begin{equation*}
\max_l \max_i |\nabla F(\mathbf{x}^{(l)})_i - \nabla f(\mathbf{x}^{(l)})_i| \le \frac{1}{8}x^*_{max}x^*_{min}
\end{equation*}
holds with probability $1-\O\left(n^{-10}\right)$. The bound (\ref{eq:toshow3}) now follows since $\mathbf{X}^0$ is close to $\x^{(I_{max})}$. We can write
\begin{align*}
\big|\nabla F(\mathbf{X}^0) - \nabla F(\x^{(I_{max})})\big| &\le \bigg|\frac{1}{m}\sum_{j=1}^mA_{ji}\big((\A_j^T\mathbf{X}^0)^3\big) - \big(\A_j^T\x^{(I_{max})}\big)^3\bigg| \\
&\quad + \bigg|\frac{1}{m}\sum_{j=1}^m A_{ji}\big(\A_j^T\x^*\big)^2\big(\A_j^T\big(\mathbf{X}^0 - \x^{(I_{max})}\big)\big)\bigg|.
\end{align*} 
As both terms can be bounded the same way, we only demonstrate the following computations for the first term. Using the definitions and H\"{o}lder's inequality, we can bound
\begin{align*}
\bigg|\frac{1}{m}\sum_{j=1}^mA_{ji}\big((\A_j^T\mathbf{X}^0)^3\big) - \big(\A_j^T\x^{(I_{max})}\big)^3\bigg| & = \bigg|\frac{1}{m}\sum_{j=1}^m A_{ji}A_{jI_{max}}^3 \frac{\hat{\theta}^3 - 1}{3\sqrt{3}}\bigg| \\
&\le \frac{1}{m}\sum_{j=1}^m \big|A_{ji}A_{jI_{max}}^3\big| \cdot \bigg|\frac{\hat{\theta}^3 - 1}{3\sqrt{3}}\bigg| \\
&\le \Biggl(\frac{1}{m}\sum_{j=1}^m A_{ji}^4\Biggr)^{1/4}\Biggl(\frac{1}{m}\sum_{j=1}^m A_{jI_{max}}^4\Biggr)^{3/4}\bigg|\frac{\hat{\theta}^3 - 1}{3\sqrt{3}}\bigg|.
\end{align*}
It follows from standard Gaussian concentration that the first two sums are bounded by $\O(1)$ with high probability. As shown above, we can bound
\begin{equation*}
\bigg|\frac{\hat{\theta}^3 - 1}{3\sqrt{3}}\bigg| \le \O\Biggl(\sqrt{\frac{\log n}{m}}\Biggr) \le \O(x^*_{max}x^*_{min}),
\end{equation*}
where we used the assumption $x^*_{min} \ge \Omega (1/\sqrt{k})$.
Repeating the same computation for the second term, we can show that 
\begin{equation*}
\big|\nabla F(\mathbf{X}^0) - \nabla F(\x^{(I_{max})})\big| \le \frac{1}{16}x^*_{max}x^*_{min}.
\end{equation*}
Recalling the definition of the population gradient $\nabla f$, we can also bound
\begin{equation*}
\big|\nabla f(\mathbf{X}^0) - \nabla f(\x^{(I_{max})})\big| \le \frac{1}{16}x^*_{max}x^*_{min},
\end{equation*}
which completes the proof of (\ref{eq:toshow3}) and therefore also completes the proof of Lemma \ref{lemma:suprec}. \hfill $\square$

\end{document}